\documentclass[nohyperref]{article}

\usepackage{microtype}
\usepackage{graphicx}
\usepackage{subfigure}
\usepackage{booktabs} 
\usepackage{enumitem}
\usepackage{listings} 

\usepackage{hyperref}

\usepackage[preprint]{icml2022}

\usepackage{amsmath}
\usepackage{amssymb}
\usepackage{mathtools}
\usepackage{amsthm}
\usepackage{svg}
\usepackage[capitalize,noabbrev]{cleveref}
\theoremstyle{plain}
\newtheorem{theorem}{Theorem}[section]

\newtheorem{lemma}[theorem]{Lemma}

\theoremstyle{definition}

\theoremstyle{remark}

\usepackage[textsize=tiny]{todonotes}

\usepackage{pifont}

\usepackage{adjustbox}

\icmltitlerunning{Investigating Power laws in Deep Representation Learning}

\begin{document}

\twocolumn[
\icmltitle{Investigating Power laws in Deep Representation Learning}

\icmlsetsymbol{equal}{*}

\begin{icmlauthorlist}
\icmlauthor{Arna Ghosh}{equal,mcgill,mila}
\icmlauthor{Arnab Kumar Mondal}{equal,mcgill,mila}
\icmlauthor{Kumar Krishna Agrawal}{equal,ucb}
\icmlauthor{Blake Richards}{mcgill,mila,mni}
\end{icmlauthorlist}

\icmlaffiliation{mcgill}{School of Computer Science, McGill University, Montr\'eal, Canada}
\icmlaffiliation{ucb}{Department of Electrical Engineering and Computer Sciences, University of California, Berkeley, USA}
\icmlaffiliation{mila}{Mila- Quebec Artificial Intelligence Institute, Montr\'eal, Canada}
\icmlaffiliation{mni}{Montreal Neurological Institute, Montr\'eal, Canada}

\icmlcorrespondingauthor{Arna Ghosh}{arna.ghosh@mail.mcgill.ca}
\icmlcorrespondingauthor{Kumar Krishna Agrawal}{kagrawal@berkeley.edu}

\icmlkeywords{Machine Learning, ICML}

\vskip 0.3in
]



\printAffiliationsAndNotice{\icmlEqualContribution} 

\begin{abstract}

Representation learning that leverages large-scale labelled datasets, is central to recent progress in machine learning.  Access to task relevant labels at scale is often scarce or expensive,  motivating the need to learn from unlabelled datasets with self-supervised learning (SSL). 
Such large unlabelled datasets (with data augmentations) often provide a good coverage of the underlying input distribution. However evaluating the representations learned by SSL algorithms still requires task-specific labelled samples in the training pipeline. Additionally, the generalization of task-specific encoding is often sensitive to potential distribution shift. Inspired by recent advances in theoretical machine learning and vision neuroscience, we observe that the eigenspectrum of the empirical feature covariance matrix often follows a power law. For visual representations, we estimate the coefficient of the power law, $\alpha$, across three key attributes which influence representation learning: \textit{learning objective} (supervised, SimCLR, Barlow Twins and BYOL), \textit{network architecture} (VGG, ResNet and Vision Transformer), and \textit{tasks} (object and scene recognition). We observe that under mild conditions, proximity of $\alpha$ to 1, is strongly correlated to the downstream generalization performance. Furthermore, $\alpha \approx 1$ is a strong indicator of robustness to label noise during fine-tuning. Notably, $\alpha$ is computable from the representations without knowledge of any labels, thereby offering a framework to evaluate the quality of representations in unlabelled datasets.

\end{abstract}

\vspace{-2em}
\section{Introduction}

\newcommand{\ftheta}{\mathbf{f}_{\theta}}
\newcommand{\covariance}{\frac{1}{N}\sum_{i=1}^{N} \ftheta(x_i) \ftheta(x_i)^\top}

Representation learning using deep neural networks (DNN) is central to recent progress in machine learning (ML). The most common approach used for learning effective representations from visual data is supervised learning, wherein parameters of a model are trained with labels to optimize performance of a particular task. However, supervised training requires large-scale human-annotated dataset, which are often expensive and difficult to collect, thereby restricting it's application to a narrow domain. Recent advances in self-supervised learning (SSL), i.e. learning useful representations without relying on labels, provide early evidence of task-agnostic representations learned by imposing structural constraints on the features \cite{chen2020simple,tian2020understanding, zbontar2021barlow}. As such, \textit{evaluating} representations in SSL systems entails  computing performance on downstream tasks while fixing the feature map (usually with linear readouts on intermediate layers). While SSL algorithms learn these representations without labelled data, assessing the quality is dependent on access to downstream labels. This motivates a natural question, \emph{is it possible to assess the quality of visual representations in deep neural networks without labels?}

\begin{figure}[t]
\includegraphics[width=\columnwidth]{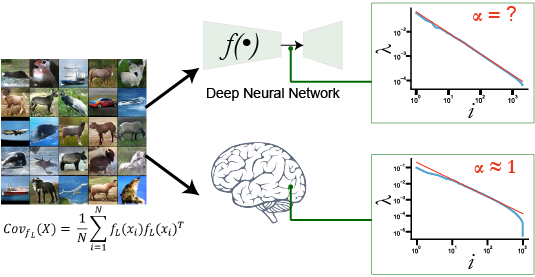}
\vspace{-2em}
\caption{Analysing large populations of neural activation in response to visual stimuli \cite{stringer2019high}, suggests an eigenspectrum with $\lambda_i \propto i^{-\alpha}$, with coefficent $\alpha \approx 1$. For intermediate representations learned by DNNs, we evaluate the sample covariance matrix $\Sigma_N(\mathbf{f})=\covariance$ and investigate the eigenspectrum of $\Sigma_N{\mathbf{(f)}}$. }
\vspace{-1em}
\end{figure}
To answer this question, we must first define desirable characteristics of ``good'' representations. Arguably, the most important attribute of a good representation is its usability for a wide range of downstream tasks. In other words, it would be ideal if we could somehow determine whether a given mapping for representations will allow us to achieve high levels of performance even when we expose the network to different data distributions. 

Notably, biological brains possess representations that are useful for a multitude of downstream tasks. What are the properties of biological neural representations that may permit this general usability? Recent findings in vision neuroscience have revealed an interesting property of learned representations in the brain \cite{stringer2019high} that may point to the answer. Specifically, representations in primary visual cortex (V1 in mice) were found to obey a power law, i.e. the variance explained by the $n^{th}$ principal component of the covariance matrix scaled as $1/n$. In recent work \cite{kong2022increasing} observe similar structure in representations from V1 in macaque monkeys and relate robustness to eigenspectrums of neural activations. Put another way, the coefficient of the power law (often denoted $\alpha$), has been observed close to 1 in biological brains. Concurrent work in DNNs has shown that enforcing $\alpha \approx 1$ in representations makes them more robust to adversarial attacks \cite{nassar20201}. Together, these results suggest that an $\alpha$ close to 1 may indicate a model that has the potential to exhibit robustness to noise and good generalization.

We explore this idea and examine the extent to which an $\alpha$ value close to 1 correlates with downstream performance on new data distributions and finetuning under loss functions. We show that across a variety of DNN architectures, learning objectives, and tasks, models that exhibit good generalization performance exhibit $\alpha$ values close to 1. Furthermore, we show that $\alpha$ values close to 1 are also correlated with good transfer learning performance. These results suggest that proximity of $\alpha$ to 1 is a potential measure of representation quality. Importantly, this measure can be calculated without any labels. As such, it provides a means to assess SSL even when labelled datasets are unavailable.

\subsection{Related Work}

\textbf{Evaluating representations and model quality}
We note a substantial body of work aiming to empirically characterize structure of emergent representations in DNN, without requiring labels \cite{nguyen2020wide,raghu2021vision}. One such index that quantifies the similarity of representations across layers (of same or different models) is Centered Kernel Alignment (CKA) \cite{kornblith2019similarity}. While CKA does not provide explicit guidance for downstream performance, \citep{martin2021predicting} show that in-distribution generalization gap can be predicted using a different index based on the model's parameters.  In particular, they show that Empirical Spectral Density (ESD) of weight matrices for many DNNs obey a power-law, with the coefficient of decay being predictive of in-distribution performance. In the present work, we explore similar indices that potentially correlate with out-of-distribution generalization by examining the eigenspectrum of \emph{activations}. 

\textbf{Generalization in Overparameterized Models}
Modern neural networks often have significantly more parameters than number of training samples, challenging classical understanding of the bias-variance tradeoff. Overparameterization permits neural networks to overfit to noise in training data, without impairing their generalization to unseen data. In recent work, Bubeck \textit{et al.} proved that for the interpolator (function implemented by the neural network) to be smooth, overparameterization is a necessary condition \cite{bubeck2021universal}. 

Furthermore, this \textit{benign overfitting} phenomenon in an overparameterized linear regression problem has been linked to the power law coefficient of the input covariance matrix \cite{bartlett2020benign}.
Specifically, Bartlett \textit{et al.} showed that for an infinite-dimensional linear regression problem, benign overfitting is possible iff the eigenspectrum satisfies a power law (upto polylog factors). 
More recently, Lee \textit{et al.} found that the tail eigenvalues of infinite-width network kernels exhibit a power law decay \citep{lee2020finite}.
Following this, Tripuraneni \textit{et al.} explored high-dimensional random feature regression settings and analytically showed a dependence between eigenspectrum decay rate of the feature covariance matrix and generalization error \citep{tripuraneni2021covariate}.  
While these characterizations provide a theoretical understanding of generalization error in the asymptotic or random feature settings, corresponding questions in the finite dimensional DNN trained with gradient descent are open problems.

For deep linear networks trained using gradient descent, the eigenvalues of input covariance determine the generalization error dynamics \cite{advani2020high}. Advani \textit{et al.} demonstrated that small eigenvalues determine the convergence of training dynamics as well as the overfitting error at convergence. For overparameterized 2-layer neural networks, Arora \textit{et al.} provided a fine grained analysis of generalization bounds \cite{arora2019fine}. 
In contrast, we study modern DNN architectures and explore the covariance structure of their learned features on visual recognition tasks.

\textbf{Main Contributions} Our core contributions include:
\vspace{-1.0em}
\begin{enumerate}
    \itemsep0em
    \item establishing that across architectures (VGG, ResNet, ViT), intermediate layers exhibit representations where the eigenspectrum follows a power law.
    \item empirical verification that the proximity of $\alpha$ to 1 is strongly correlated with downstream performance across multiple key attributes (backbone architecture, pretraining objective \& downstream task) which influence representation learning. 
    \item demonstrating that self-supervised learning algorithms learn representations that are robust to finetuning under noisy labels. Additionally, improvement in performance across finetuning correlates strongly with the proximity of $\alpha$ to 1.
\end{enumerate}

\section{Preliminaries}

We are interested in evaluating quality of features learned by modern neural network architectures, especially in the high-dimensional space of visual representations. Formally, we consider inputs $\mathbf{x}_i \in \mathcal{X} \subset \mathbb{R}^d$, and learn mappings $\mathbf{f} : \mathbb{R}^d \to \mathbb{R}^D$ such that $\mathbf{f}(\mathbf{x})$ is a vector of $D$-dimensional features. For instance, sample images from the ImageNet dataset have dimensions $(3, 224, 224)$, i.e $d \approx 0.15M$.

We consider DNNs as our function approximators, where each architecture implicitly defines a function class $\mathcal{F} = \{ \mathbf{f}_{\theta} : \theta \in \Theta \}$ where $\Theta \subset \mathbb{R}^p$ is the feasible set of model parameters (e.g bounded $\Theta, {[-B, B]}^p$ for some $B \in \mathbb{R})$. The search for \emph{good} representations usually poses the following optimization problem:

\emph{With a dataset $\mathcal{D}_{\mathrm{pretrain}}$ from some data distribution $\mathbb{P}_{\mathrm{pretrain}}$ (potentially with labels), search for \emph{optimal} parameters $\theta^*$ such that 
\begin{align}
    \theta^* = \arg\min_{\theta \in \Theta} \mathcal{L}(\mathbf{f}_\theta, D_{\mathrm{pretrain}})
    \label{eq:pretrain}
\end{align}
\label{eq:DNN_opt}
where $\mathcal{L}(\mathbf{f}, D)$ is the learning objective.}

The above optimization problem is usually non-convex, and often use gradient based optimizers to find an approximate solution $\hat{\theta}$. With $\hat{\theta}$, we evaluate the quality of representations on a downstream task with $\mathcal{D}_{\mathcal{T}}$ of $N$ samples from $\mathbb{P}_{\mathcal{T}}$. 
\begin{align}
    R(\hat{\theta}, \mathcal{T}) = \mathcal{L}(\mathbf{f}_{\hat{\theta}}, D_{\mathcal{T}})
\end{align}
A concrete example is the pretraining on ImageNet dataset, where we use a pretrained VGG16 model to extract features from images. These representations are finetuned with a linear readout on MIT-67 dataset.

\subsection{Covariance estimation and eigenspectrum}
\label{subsec:alpha_estimation}

For parameterized functions $\mathbf{f}_{\theta} : \mathcal{X} \to \mathbb{R}^D$ , the empirical feature covariance matrix, $\Sigma_N(\mathbf{f_\theta})$ is defined as 
\begin{align*}
    \Sigma_N(\mathbf{f_\theta}) &= \frac{1}{N} \sum_{i=1}^N\mathbf{f}_\theta(x_i) \mathbf{f}_\theta(x_i)^T 
\end{align*}
The eigenspectrum of  $\Sigma_N(\mathbf{f_\theta})$ informs us about the variance explained by each principal component of the space spanned by representations $\mathbf{f_\theta(x)}$.
Using the spectral decomposition theorem on symmetric matrices, $\Sigma = U \Lambda U^T$, where $\Lambda$ is a diagonal matrix with nonnegative entries, and $U$ is a matrix whose columns are the eigenvectors of $\Sigma$. Without loss of generality we assume that $\lambda_1 \geq \lambda_2 ... \geq \lambda_m$, where $m = \min(N, D)$ is the rank of $\Sigma_N(\mathbf{f_\theta})$.

The eigenspectrum of a covariance matrix is said to follow a \emph{power law} $PL(\alpha)$, or \emph{zeta distribution} if for $\lambda_j \in [\lambda_{min}, \lambda_{max}]$, the eigenvalues $\lambda_j$ are all nonnegative, and 
\begin{align*}
    \lambda_j \propto j^{-\alpha}
\end{align*}
for some $\alpha > 0$. Here, $\alpha$ is the \emph{slope} of the power law, and is referred to as the \emph{coefficient of decay} of the eigenspectrum. Intuitively, small $\alpha$ (typically $\alpha \leq 1$) suggests a dense encoding, while a high $\alpha$ (rapid decay) corresponds to a sparse encoding.


\subsection{Network Architectures}

We examine representation learning primarily in the context of visual downstream tasks. As such, we look at neural network architectures which follow different design principles, with the shared goal of pushing the envelope on generalization and robustness.

\textbf{Classic Convolutional Neural Networks}  Deep Convolutional Neural Networks (CNNs) have been instrumental in early successes in image classification. In particular, \cite{simonyan2014very} proposes a class of VGG-X architectures, characterized by a sequence of convolutional layers, two fully connected layers and a softmax readout layer. 

\textbf{Residual Networks with Skip Connections} Learning with residual connections \cite{he2016deep} propose a breakthrough in the ability to train very deep neural networks ($\sim$1000 layers) by adding a residual connection to \emph{block of convolutions}. Recent findings of \cite{raghu2021vision} suggest that early and later residual blocks learn qualitatively different representations of the same input. In particular, the early layers are shown to learn more localized information about the input, while the downstream layers with a larger receptive field are able to capture more global information.

\textbf{Vision Transformers} Recent advances with self-attention modules in transformers \cite{dosovitskiy2020image} propose a different paradigm for representation learning. With patch-based image embedding and self-attention, even the early layers have a global receptive field, allowing the model to store representations at different scales via multiple heads. While the later layers of the network learn specialization to have globally relevant features, \cite{raghu2021vision} suggests that the representations learned by the early layers are indeed across different scales.

\subsection{Learning objectives}

As outlined in \cref{eq:pretrain}, models in the same function class are often optimized under different learning objectives, to incorporate inductive biases. We broadly differentiate between supervised (with labels) and self-supervised (without labels) pretraining objectives.

\textbf{Supervised Learning} Since we consider primarily classification tasks, models trained with this objective learn to predict the labels corresponding to inputs in $\mathcal{X}$. To recover the maximum likelihood estimator, the cross entropy loss is often used to measure the error of the model, as $\mathcal{L}_{\mathrm{sup}} = -\frac{1}{N} \sum_{i=1}^N y_i \log p_\theta (x_i)$ where $p_\theta(x_i) = \mathrm{softmax}(\mathbf{g}_{\phi}(\mathbf{f_\theta}(x_i)))$ denotes the logits and $\mathbf{g}_{\phi} : \mathbb{R}^D \to \mathbb{R}^{k}$ denote the readout layer. Here, $k$ is the number of classes.

\textbf{Self-Supervised Learning} A common theme in several recently proposed SSL algorithms is to impose invariance constraints on the representations, and rely on different views of the data to learn these embeddings. In our experiments, we evaluate features pretrained under two classes:

\begin{enumerate}
\item \textbf{Dual-Networks} Motivated by the invariance of representations under benign augmentations, the \emph{dual-networks} learning objective can be summarised as 
\begin{equation}
    \mathcal{L}_{\mathrm{dual}} = \frac{1}{2} \| \mathbf{f}(\mathbf{x}^{A}; \theta) - \mathbf{f}(\mathbf{x}^{B}; \theta_0) \|^2 
\end{equation}
In particular, we consider SimCLR, BYOL \cite{chen2020simple, grill2020bootstrap} as pretraining objectives, and evaluate the representations on downstream task performance. 
\item \textbf{Efficient Encoding}
Inspired by the efficient coding hypothesis \cite{barlow1961possible}, the \emph{Barlow Twins} learning objective \cite{zbontar2021barlow} proposes imposing a \emph{soft-whitening} constraint as the learning objective.
\begin{align}
    \mathcal{L}_{\mathrm{BT}} &= \underset{\text{invariance}}{\underbrace{\sum_i (1-\mathcal{C}(\mathbf{f_\theta})_{ii})^2}} + \lambda  \underset{\text{redundancy-reduction}}{\underbrace{\sum_i \sum_{j \neq i} \mathcal{C}(\mathbf{f_\theta})_{ij}^2}} \\
    \text{where } & \mathcal{C}(\mathbf{f_\theta})_{ij} = \frac{\sum_{N} \mathbf{f}(x^A; \theta)_i \mathbf{f}(x^B; \theta)_j }{ \sqrt{\sum_{N} \mathbf{f}(x^A; \theta)_i^2 \sum_{N} \mathbf{f}(x^B; \theta)_j^2} } \nonumber
\end{align}

Notably with sufficient large $\lambda$, the model would impose an $\alpha=0$ constraint on the representations.
\end{enumerate}
\section{Power-Laws in Deep Representation Learning}

\begin{algorithm}[tb]
    \caption{Computing Eigenspectrum Decay}
    \label{alg:spectral_decay}
    
     \definecolor{codeblue}{rgb}{0.25,0.5,0.5}
     \lstset{
       basicstyle=\fontsize{8pt}{8pt}\ttfamily\bfseries,
       commentstyle=\fontsize{8pt}{8pt}\color{codeblue},
       keywordstyle=\fontsize{8pt}{8pt},
     }
 \begin{lstlisting}[language=python, mathescape=true]
 # model: encoder network
 # layer: layer index
 # B: batch size
 # N: size of evaluation split
 # D: dimensionality of the representations

 cov = torch.zeros(D, D)
 
 for batch in eval_dataloader: 
     # extract features 
     feats = model.feats(batch,layer=layer) # BxD

     # aggregate covariates per batch
     cov += torch.mm(feats.T, feats) / N # DxD

 # compute eigenspectrum 
 eigenspectrum = torch.linalg.eigvals(cov)

 # find decay coefficient (where $\lambda_i \propto i^{-\alpha}$)
 alpha = fit_powerlaw(eigenspectrum)
 \end{lstlisting}
 \end{algorithm}

Incorporating structure in deep representation learning (e.g inductive biases via model architectures) has been extremely effective in reducing generalization error of learned models across domains. To explain these empirical findings, recent work in theoretical machine learning attempts to analyse bounds on generalization error in restricted settings. While we primarily focus on finite-width neural networks on real-world datasets, we begin by highlighting  generalization bounds in the infinite feature dimensional regime \cite{bartlett2020benign} for linear regression with Gaussian features. We extend such properties of the min-norm interpolant to finite-width linear regression with power-law in covariates, optimized using gradient descent to motivate our experiments in neural networks. 
For the rest of the section, we denote feature maps as $\mathbf{f}_\theta(x) \in \mathbb{R}^D$ where $\mathbf{f}_{\theta} : \mathcal{X} \to \mathbb{R}^D$, and the readout network as $\mathbf{g}_{\phi} : \mathbb{R}^D \to \mathbb{R}^{k}$, where $k$ is target dimensionality. For simplicity of analysis, we consider linear readouts unless explicitly mentioned, i.e. $\mathbf{g}_{\phi}(x) = x^T\phi$.

\subsection{$D\to\infty$ regime}
The eigenspectrum of the feature covariance matrix sheds light on the relation between the smoothness of mapping $\mathbf{f}_\theta$ and good generalization in the asymptotic regime when $D \to \infty$. In this setting, Stringer \textit{et al.} show that the kernel function associated with $\mathbf{f}_{\theta}$ is continuous iff $\alpha > 1$ (see Thm 3. in Supplementary of \cite{stringer2019high} for proof).

\begin{lemma}
\cite{stringer2019high} Let $\lambda_1 \geq \lambda_2 \geq ... \geq 0$ be the eigenspectrum of $\Sigma_{\mathbf{f}}$ and $K(x,x')$ be the kernel function corresponding to the mapping $\mathbf{f}_{\theta} : \mathcal{X} \to \mathbb{R}^D$. If $\frac{1}{D}\mathbb{E}_{x}[\mathbf{f}_{\theta}(x)^T\mathbf{f}_{\theta}(x)]$ is finite, ie. finite variance in the representation space, and $K$ is continuous, then $\lambda_n = o(n^{-1})$.
\label{theorem_Stringer_continuity}
\end{lemma}
\vspace{-1em}

The proof for \cref{theorem_Stringer_continuity} entails showing that trace of the continuous kernel $K$ is the total variance in representation space, which in turn equals trace of $\Sigma({\mathbf{f}})$, i.e. sum of all eigenvalues, $\lambda_i$'s. Following this equality, it can be shown that for finite variance, $\lambda_n \leq \epsilon\frac{1}{n}$ for all $n \geq N$ and some $\epsilon > 0$. Interestingly, \cref{theorem_Stringer_continuity} shows a direct relation between $\alpha$ and smoothness of the representation kernel and offering an insight into why having slow eigenspectrum decay (or low decay coefficients) could be pathological. 

In more recent work, \cite{bartlett2020benign} studied the linear regression setting with Gaussian features, and proved that the min-norm solution provides good generalization performance iff the eigenspectrum of $\Sigma(\mathbf{f_\theta})$ follows a power law (upto polylogarithm factors) with $\alpha=1$. More formally,

\begin{lemma}
\cite{bartlett2020benign}
In a linear regression problem parameterized as $\hat{y} = \mathbf{f}_{\theta}(x)^T\psi$, let the risk of a min-norm solution be defined as
\begin{equation}
    R[\psi] := \mathbb{E}_{x,y}\left[(y-\mathbf{f}_{\theta}(x)^T\psi)^2 - (y-\mathbf{f}_{\theta}(x)^T\psi^{*})^2 \right]
\end{equation}
where $\psi^{*}$ denotes the optimal parameter vector. If the $n^{th}$ eigenvalue of $\Sigma({\mathbf{f}})$ follows a power law, i.e. $\lambda_n = n^{-\alpha}\ln (n+1)^{-\beta}$, then $R[\psi]$ is small iff $\alpha=1$ and $\beta > 1$ 
\label{theorem_Bartlett_benign}
\end{lemma}

We refer the reader to \cite{bartlett2020benign} for the proof of  \cref{theorem_Bartlett_benign}.
Taken together, \cref{theorem_Stringer_continuity} and \cref{theorem_Bartlett_benign} offers a normative explanation for $\alpha$ being close to 1, i.e. for representations to be smooth as well as offering a basis for achieving competitive performance on downstream tasks. 

\begin{figure}
    \centering
    \includegraphics[width=0.41\textwidth]{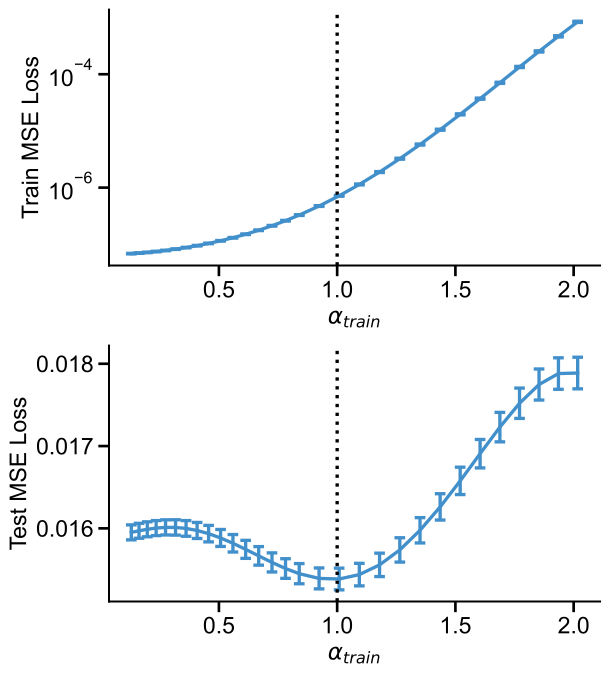}
    \caption{Overparameterized linear regression, with inputs drawn from Gaussian distribution with power-law in the covariance matrix. Models are trained with gradient descent with initialization $\psi_0 = \mathbf{0}$. Note that $\alpha > 1$ particularly suffers from high train, test MSE loss, with low generalization error with $\alpha \approx 1$.}
    \label{fig:regression}
    \vspace{-1em}
\end{figure}
\subsection{Finite dimensional models}
The asymptotic regime of infinite width is an excellent framework to study theoretical properties of DNN representations. Despite having a large number of parameters, practical DNNs always possess finite dimensional representations, making it important to investigate the implications of such results in finite width models. As such, \cref{theorem_Stringer_continuity} is hinged on the fact that the total variance in the representations is bounded and equals the sum of eigenvalues, under continuous kernels. In finite dimensions, the sum of the eigenvalues will be bounded as long as $\alpha \geq 0$, so $\alpha$ might not necessarily be indicative of the geometry in the representations. Similarly, \cref{theorem_Bartlett_benign} allows for a wide range of eigenvalue sequences alongside powerlaws with $\alpha > 1$(refer to Thm 6 in \cite{bartlett2020benign} for details). 

These finite dimensional analogies to \cref{theorem_Stringer_continuity} and \cref{theorem_Bartlett_benign} raise the question: \emph{Does $\alpha$  sufficiently larger than 1, still allow efficient learning and strong generalizability?}

To answer this question, we narrow our focus to gradient-based optimization techniques usually used to train DNNs. In particular, from the optimization perspective, Advani \textit{et al.} showed that for deep linear regression in high dimensions, the time required for training and the steady-state generalization error are both $\mathcal{O}(\frac{1}{\lambda_{min}})$ \cite{advani2020high}. A key difference from our work is that they assume the inputs are drawn from an isotropic Gaussian distribution. Instead, we investigate the generalization of linear regression on $\mathbf{f}_{\theta}(x)$, which has a power law structure in its covariates. First, in the linear regression setting, Advani \textit{et al.} emphasize
\begin{lemma}
Let $\hat{y} = \mathbf{f}_{\theta}(x)^T\psi$ be a finite dimensional linear regression problem where $\psi$ is learned using gradient descent in order to optimize the training error, $\mathbb{E}_{x,y}[(y - \mathbf{f}_{\theta}(x)^T\psi)^2]$, where $(x,y) \sim \mathcal{D}_{train}$, i.e. the training dataset. If we assume isotropic Gaussian noise in targets $y$, then the time required by gradient descent to minimize the training error, $T_{convergence} = \mathcal{O}(\frac{1}{\lambda_{min}})$
\label{theorem_advani_saxe_dynamics}
\end{lemma}

\cref{theorem_advani_saxe_dynamics} effectively states that small eigenvalues in $\Sigma_{\mathbf{f}}$ impedes training using gradient descent. We extend \cref{theorem_advani_saxe_dynamics} to explicitly state training convergence time as the following theorem:

\begin{theorem}
Let $\hat{y} = \mathbf{f}_{\theta}(x)^T\psi$ be an overparameterized linear regression problem where $\psi$ is learned using gradient descent in order to optimize the training error, $\mathbb{E}_{x,y}[(y - \mathbf{f}_{\theta}(x)^T\psi)^2]$, where $(x,y) \sim \mathcal{D}_{train}$. If we assume power law distribution in eigenspectrum of representations at $\mathbf{f}_{\theta}$, i.e. $\lambda_n = \frac{c}{n^{\alpha}} \quad \forall n \geq n^{*}$, where $n^{*} \in \{1,2 ... N\}$, then the time required by gradient descent to minimize the training error, $T_{convergence} = \mathcal{O}(N^{\alpha})$
\label{theorem_us_convergence} where N is number of training samples.
\end{theorem}

The outline of the proof builds on key results relating to gradient descent dynamics from \cite{shah2018minimum}. We show that gradient descent updates, when $\psi$ is initialized to 0, yield a recursive relation for $\psi(k)$, i.e. $\psi$ after $k$ update steps. Plugging this relation in the gradient formulation, we show that the update step length along the $n^{th}$ principal direction of $\mathbf{f}_{\theta}(x)$ shrinks exponentially with a decay rate proportional to $\lambda_n$. Therefore, the time to convergence in training is controlled by the smallest eigenvalue which, by design, follows the power law. In sum, \cref{theorem_us_convergence} provides an explanation against arbitrarily large values of $\alpha$. Taken together, \cref{theorem_us_convergence} suggests that $\alpha = 1$ might be beneficial, where these representations form basis for gradient-based optimization on downstream task performance.

\textbf{An Additional Motivating Example}
Before exploring the link between $\alpha \approx 1$ and generalization performance of deep networks, we first consider its relationship to the finite dimensional regression setting as in \citep{bartlett2020benign}.  Specifically, Theorem 6 of their paper states that the conditions for ``benign overfitting'', wherein a model can perfectly fit noisy training data without any subsequent loss of performance on testing data, may be looser in finite dimensions as opposed to the necessary and sufficient condition of $\alpha=1$ in infinite dimensions. 
To empirically test this in the finite high dimensional setting, we examine linear least squares regression using different covariate structures for the input data. Formally, we consider covariates $\{x_i\}_{i=1}^N$, such that $x_i \in \mathbb{R}^d$ is sampled from a Gaussian distribution with covariance structure $\Sigma=\mathrm{diag}\{ \lambda_1, ... \lambda_d \}$ where
$\lambda_j \sim PL(\alpha)$, i.e $\lambda_j \propto c j^{-\alpha}$. 
We assume access to the corresponding labels $\{y_i\}_{i=1}^N$ generated under a teacher function $\theta^*$, such that $y_i = x_i^T \theta^* + \epsilon_i$. We find that in this scenario there is a clear relationship between the proximity of $\alpha$ to 1 and the presence of benign overfitting. As shown in \cref{fig:regression}, when $\alpha$ is close to 1, the training loss is low, but the validation loss is also low. Thus, when $\alpha$ is close to 1, the generalization properties are at their best. This example thus provides another hint that $\alpha \approx 1$ is a potential measure for how well a model will be able to generalize. 

In our experiments, we investigate the following questions on vision classification tasks, when $\mathbf{f}_\theta$ is a neural network:
\begin{enumerate}
    \itemsep0em 
    \item \textbf{Representation Quality}: How does $\alpha$ vary across backbone architectures and pretraining learning objectives? Are representations with $\alpha \approx 1$ likely to enjoy better out-of-distribution performance?
    \item \textbf{Generalization across tasks}: Is $\alpha$ informative of generalization when evaluated on different downstream tasks? 
    \item \textbf{Robustness}: On finetuning with noisy labels, does task-performance correlate to $\alpha$?
\end{enumerate}
\vspace{-1em}

\section{Experimental Setup}

Driven by this observation in linear regression settings, we investigate whether $\alpha \approx 1$ is a good indicator of generalization performance in DNNs.  We investigate this relationship between $\alpha$ and downstream task performance across different network architectures and pretraining loss functions. In this work, we restrict our focus to visual learning tasks, specifically object recognition and scene recognition tasks. 

In each experiment, we evaluate the $\alpha$ of emergent representations, $\mathbf{f}(x)$, at intermediate layers of a DNN pretrained on ImageNet \citep{5206848} by computing the covariance matrix, $\Sigma_{\mathbf{f}}$, and fitting a power law on the eigenspectrum (refer to \cref{alg:spectral_decay}). Our pretrained models are taken from PyTorch Hub \citep{NEURIPS2019_9015} and \texttt{timm} \citep{rw2019timm}. Given that we observed no significant difference between the observed $\alpha$ values in the train and test sets, we refer to this empirical estimate as the $\alpha$ for the dataset.
To estimate the capacity of intermediate representations in solving the downstream task, we train a linear readout layer, $\mathbf{g}(.)$, from representations to target logits. Intuitively, this comes down to establishing a relationship between the manifold geometry and linear separability of representations \citep{chung2018classification}. Thereafter, we observe the correlation between estimated $\alpha$ and the linear readout performance. Notably, we also tried non-linear $\mathbf{g}(.)$ and observed a similar trend in results (see Appendix).

\vspace{-0.5 em}
\subsection{Feature Backbones}

In this section, we investigate the relationship between $\alpha$ of the representation covariance matrix and object recognition performance in DNNs with different backbones and the role of depth. In order to do so, we examine varying depth configurations within network architectures across three generations of models on the STL-10 dataset \cite{5206537}. First, we observe deep Convolutional Neural Networks (CNNs) without any residual connection as our first family of models. Specifically, we choose three different configurations of VGG-Net \cite{simonyan2014very}, namely VGG-13, VGG-16 and VGG-19. We inspect representations that are input to the dropout and MaxPool layers during the forward pass of the network. Second, we consider Deep Residual Networks \cite{he2016deep} which are widely used in computer vision. We inspect the representations that are input to each of the residual blocks as well as the Adaptive average pool in ResNet-13, ResNet-50 and ResNet-101 during their respective forward passes. Finally, owing to the recent success of transformers in object recognition tasks we consider Vision Transformers (ViT) \cite{dosovitskiy2020image} as the third family of models, namely ViT-Base/8 , ViT-Large/16 and ViT-Huge/14. Unlike VGG and ResNet, we only look at features corresponding to the [CLS] token in the intermediate layers because it summarizes the entire input image and is used in practice for class prediction. For all model architectures, we use the weights obtained from pretraining on ImageNet.
 
\begin{figure}[t!]
\includegraphics[width=\columnwidth]{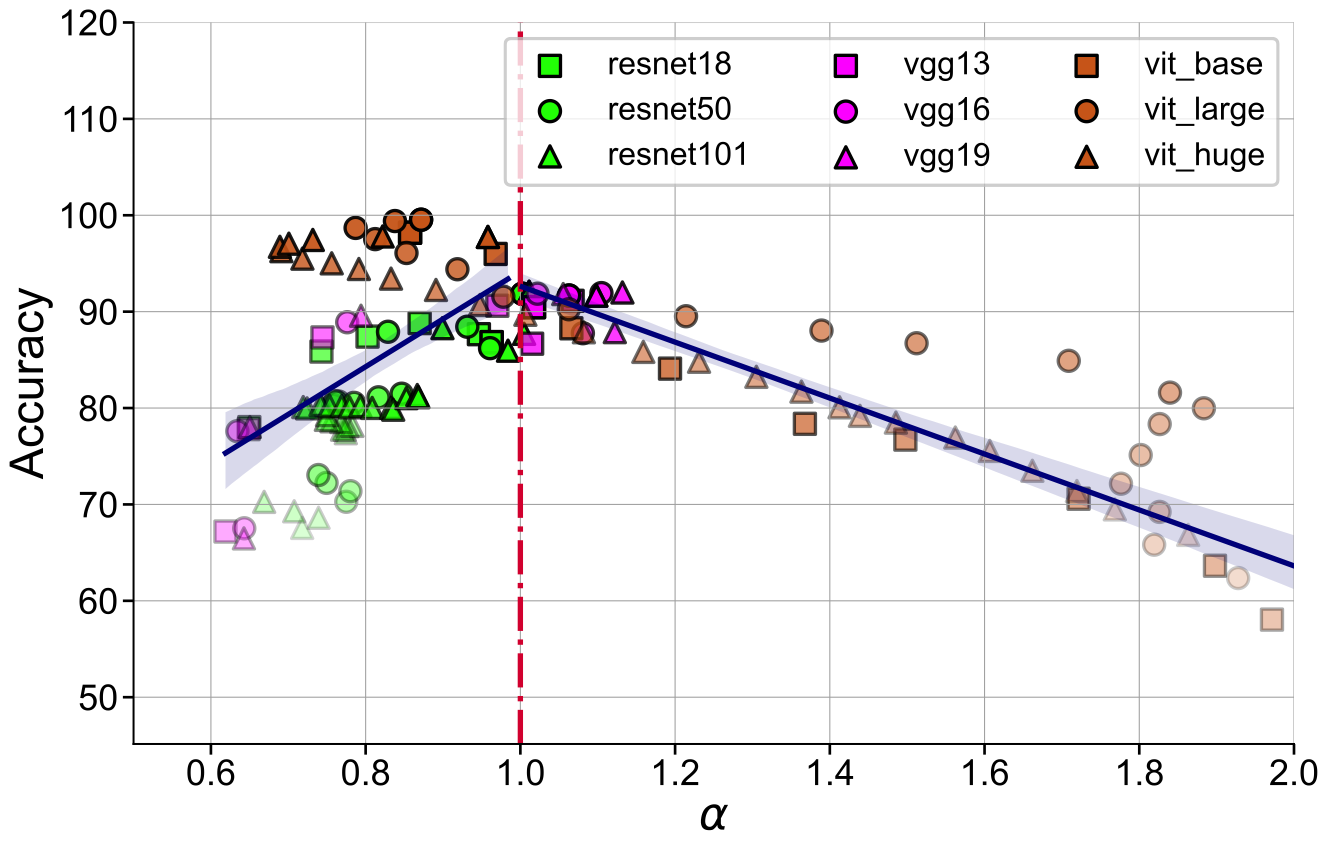}
\vspace{-2em}
\caption{Performance on STL-10 improves for representations with $\alpha$ approaching 1 across different model architectures. In each model, solid points correspond to the later layers. Correlation coefficient for the greater than 1 and less than 1 regimes are $\rho_{>1}=-0.922$, $^{*}p<0.05$ and $\rho_{<1}=0.493$, $^{*}p<0.05$ respectively.}
\vspace{-1em}
\label{fig:acc_alpha_arch}
\end{figure}

\cref{fig:acc_alpha_arch} illustrates the relation between performance and $\alpha$ for all the nine architectures across intermediate layer representations, as described above. We found that, while most intermediate representations in CNNs (with or without residual connections) exhibit $\alpha < 1$, representations in ViTs mostly exhibit $\alpha > 1$. Nevertheless, representations extracted from the deepest layers of all the models exhibit $\alpha$ value in the proximity of 1, irrespective of the total depth of each model (see Appendix \cref{fig:alpha_layers_arch}). Furthermore, the performance on downstream task increases with depth. This is unsurprising because all networks were trained to perform object recognition on ImageNet \cite{5206848} and thereby would have leveraged hierarchical processing to learn features that are tuned towards object recognition. Surprisingly enough, we observe a strong significant correlation between $\alpha$ and performance on the STL-10 dataset, i.e. a different data distribution than the training dataset, across layers and model architectures ($\rho=-0.922$, $^{*}p<0.05$ for representations exhibiting $\alpha>1$ and $\rho=-0.922$, $^{*}p<0.05$ for representations exhibiting $\alpha<1$). It is worth noting here that the correlation was weaker for the earliest layers of each model. We believe that early layers learn more task invariant features that reflect the statistics of natural images \cite{kornblith2019similarity,zeiler2014visualizing} and therefore lack task relevant information in their representations. Taken together, this observation confirms our hypothesis that $\alpha$ is a good indicator of out-of-distribution generalization performance when representations possess task relevant information.
\vspace{-0.5em}

\subsection{Learning objective}
In this section, we first aim to understand how the $\alpha$ value changes across the layers of a fixed architecture DNN when trained with different learning objectives. We take a ResNet-50 model \cite{he2016deep} pre-trained using three different SSL algorithms, namely SimCLR \cite{chen2020simple}, BYOL \cite{grill2020bootstrap} and Barlow Twins \cite{zbontar2021barlow}, and the supervised learning loss objectives on ImageNet-1k\cite{5206848} dataset. We use a similar procedure as before to extract representations from the network and estimate $\alpha$. 

\begin{figure}[t!]
\includegraphics[width=\columnwidth]{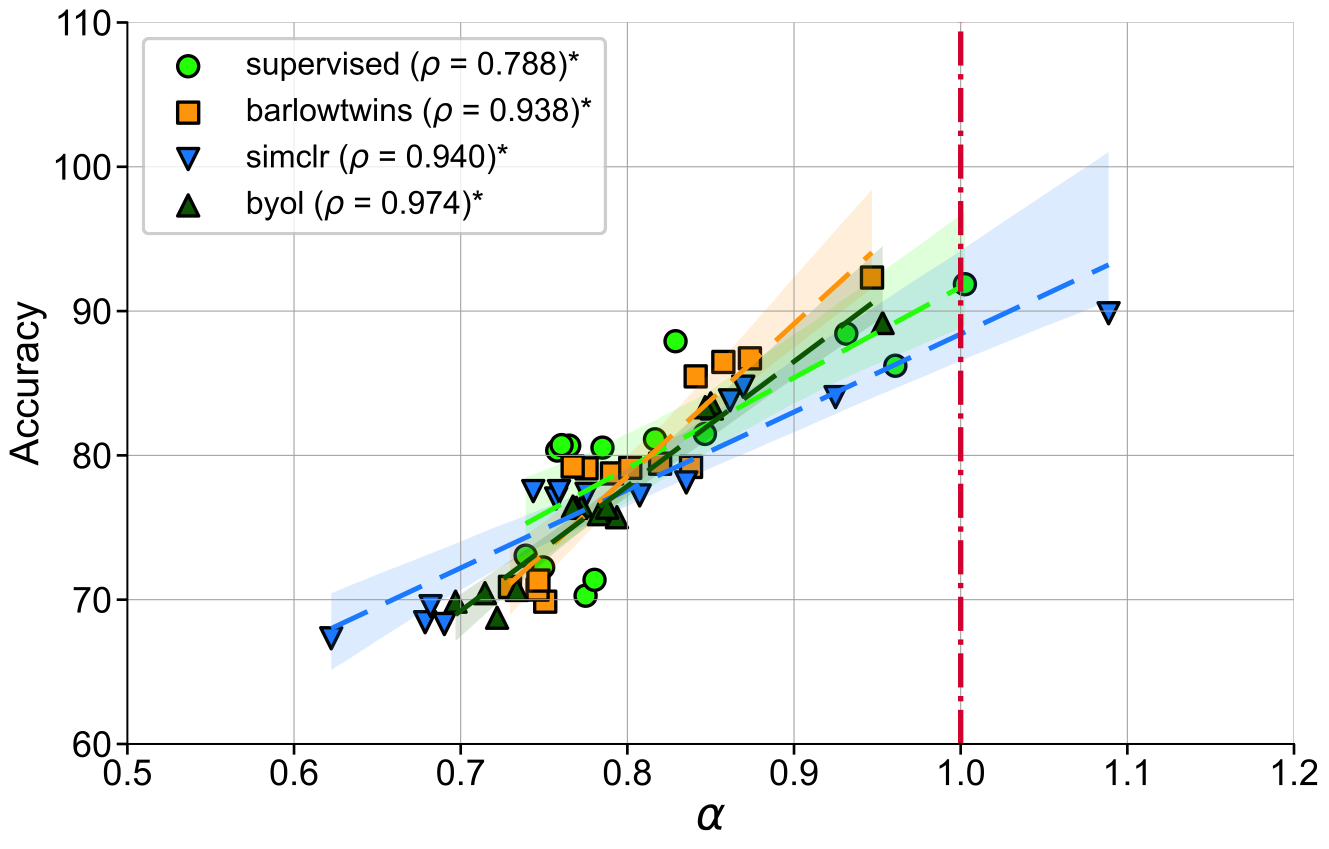}
\vspace{-2em}
\caption{The performance using linear readouts on STL-10 is strongly correlated to proximity of $\alpha$ to 1. $^{*}p<0.05$. (All models have ResNet-50 backbone.)}
\label{fig:acc_alpha_loss}
\end{figure}

Similar to results in the previous section, all networks irrespective of the pretraining loss function, exhibit $\alpha$ closer to 1 in the deeper layers in contrast to intermediate layers (see Appendix \cref{fig:alpha_layers_loss}). This surprising result indicates that although the pretraining loss function was different, representations extracted from deepest layers are reflective of the object semantics in natural images. Furthermore,  \cref{fig:acc_alpha_loss} illustrates the strong correlation between $\alpha$ and generalization performance on STL-10 across all pretraining loss functions. Together with results from the previous section, we validate our hypothesis that the representations that demonstrate good out-of-distribution generalization performance are characterized by $\alpha$ close to 1.

\begin{figure}[t!]
\includegraphics[width=\columnwidth]{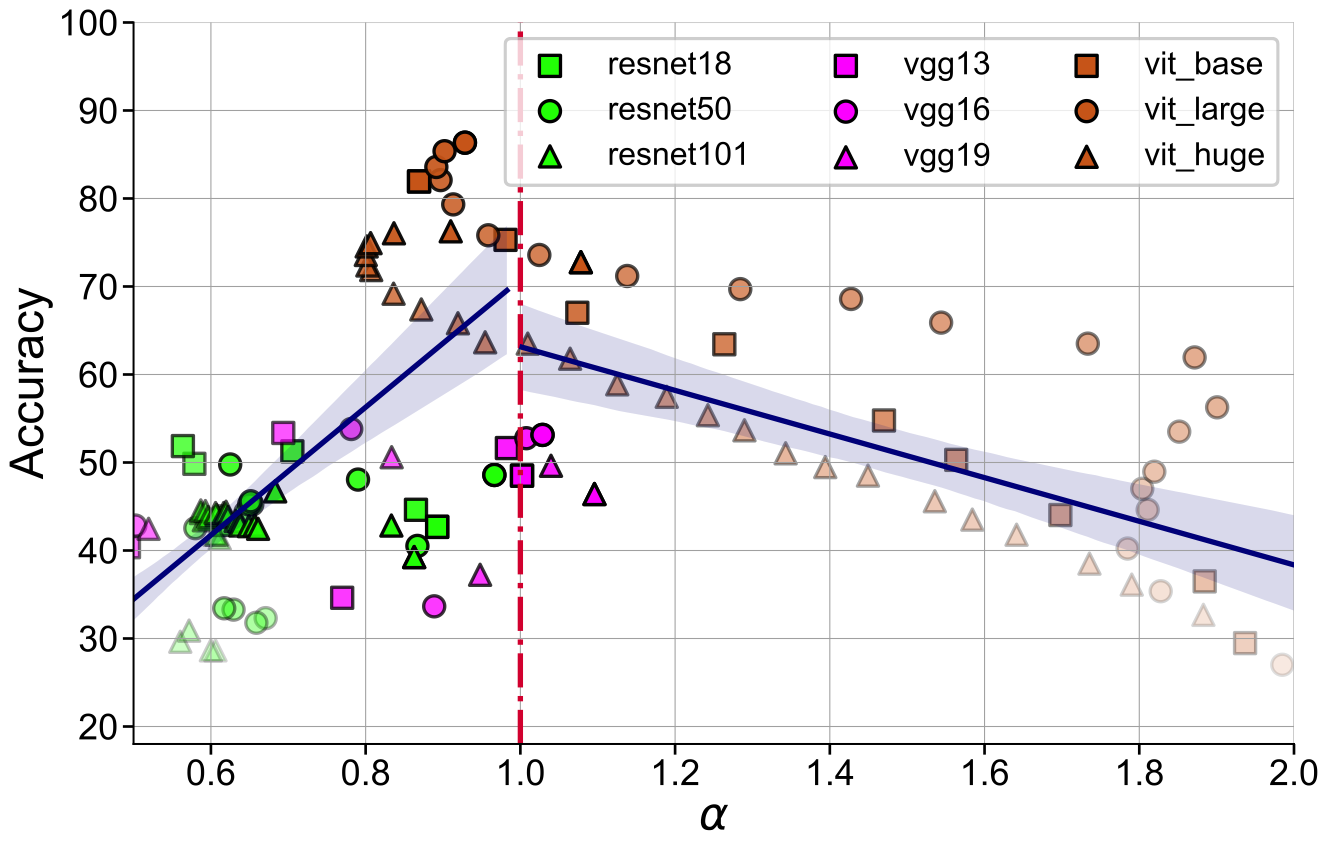}
\vspace{-2em}
\caption{Performance on MIT67 is strongly correlated to proximity of $\alpha$ to 1 across different model architectures, except ResNet models. Correlation coefficient for the greater than 1 and less than 1 regimes are $\rho_{>1}=-0.667$, $^{*}p<0.05$ and $\rho_{<1}=0.708$, $^{*}p<0.05$}
\label{plot:MIT67_loss_alpha_arch}
\vspace{-1em}
\end{figure}

\begin{figure}[t!]
\includegraphics[width=\columnwidth]{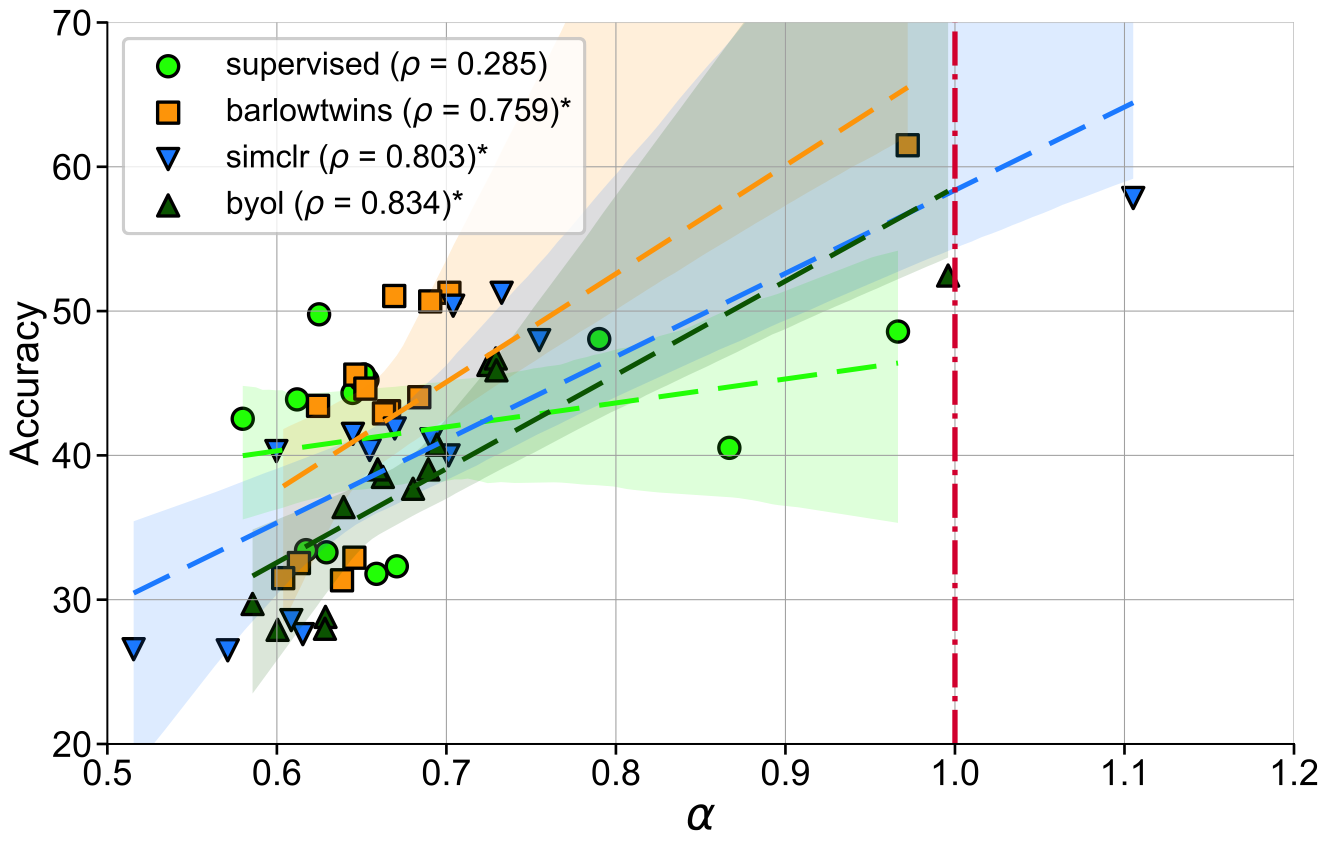}
\vspace{-1em}
\caption{For fixed backbone architecture (ResNet-50), models trained with different learning objectives (supervised, self-supervised). We evaluate intermediate features with linear probe on MIT67, observing a strong performance when $\alpha \propto 1$. $^{*}p<0.05$.}
\vspace{-1em}
\label{plot:MIT67_loss_objectives}
\end{figure}

\subsection{Scene vs Object Recognition}
So far, we have observed that $\alpha$ serves as a good indicator for out-of-distribution generalization performance, under mild considerations\footnote{When representations contain task relevant information}, when we change two key components of \cref{eq:DNN_opt}: $\mathbf{f}_\theta$ and $\mathcal{L}$. In this section, we change the evaluation task from object recognition to scene recognition. It is known that DNNs perceive objects and scenes differently based on their architecture \cite{nguyen2020wide}. We investigate whether $\alpha$ is a good measure to observe when the downstream task is different from the training task. Scene recognition is fundamentally different from the object classification task as the model needs to focus on global features in the entire image in contrast to a local features in the image which contains the object \cite{olivachapter}. We follow the same procedure as the above sections, but on the MIT67 indoor scene recognition dataset \cite{5206537}.

\cref{plot:MIT67_loss_alpha_arch} and \cref{plot:MIT67_loss_objectives} demonstrate that the relation between $\alpha$ and out-of-distribution generalization performance holds for most models on the scene recognition task. An outlier is the supervised ResNet family of models. This behavior can be attributed to the extent of local information possessed by Residual networks, owing to their fully convolutional architecture, which does not enable the resultant representations to capture information relevant to scene recognition \cite{raghu2021vision}. In other words, representations of ResNet models lack the task relevant information. Taken together, our results demonstrate that $\alpha$ can be used as a measure to estimate generalization performance across all three key elements of DNN models when the representations possess information relevant to downstream task. 

\begin{table}[t]
\caption{Classification accuracies on STL-10, on finetuning with noisy-labels (15\% label-noise) in ResNet-50.}
\label{table:robustness_STL10}
\vskip 0.15in
\begin{center}
\begin{small}
\begin{sc}
\begin{adjustbox}{width=0.47\textwidth}
\begin{tabular}{lcccr}
\toprule
& Finetune Acc. & Finetune Acc. &  \\
Objective & (noiseless) $\uparrow$ & (noise=15\%) $\uparrow$ & \% Drop $\downarrow$  \\
\midrule
BarlowTwins   & 91.81 $\pm$ 1.4 & 88.45$\pm$ 0.34 & $\sim$ 3.36 \\
SimCLR  & 88.83 $\pm$ 0.18 & 85.6$\pm$ 0.2 & $\sim$ 3.2 \\
Supervised & 88.39 $\pm$ 0.41 & 81.19 $\pm$ 1.01 & $\sim$ 8.20 \\
\bottomrule
\end{tabular}
\end{adjustbox}
\end{sc}
\end{small}
\end{center}
\vskip -0.1in
\vspace{-1.5em}
\end{table}
\subsection{Robustness to finetuning with noisy labels}

Ideally, ``good'' representations should exhibit robustness under finetuning with noisy labels. To test the robustness of visual representations learned under different objectives, and its relation to $\alpha$, we extract features from pretrained models and finetune them by training on a dataset with noisy labels. We evaluate representations that emerged when training with different learning objectives, but with a fixed backbone (ResNet-50) and the same downstream dataset (STL-10). Concretely, we extract features from layer-75 of the ResNet-50 (ResNet-50/L75) by freezing gradient propagation upto layer-75, and finetune the later layers for image classification.

In \cref{plot:finetune_STL10}, we plot the accuracy against $\alpha$ across multiple epochs of finetuning, where $\alpha$ is evaluated for the layer just before linear readouts (post-adaptive pooling). We find that as performance on the validation set improves across epochs, $\alpha$ approaches 1. When compared to finetuning with noiseless targets, \cref{table:robustness_STL10} highlights the relative robustness of self-supervised learning objectives when compared with supervised learning. This finding supports our hypothesis on variability in $\alpha$ during finetuning correlates with performance improvement.

\begin{figure}[ht!]
\includegraphics[width=\columnwidth]{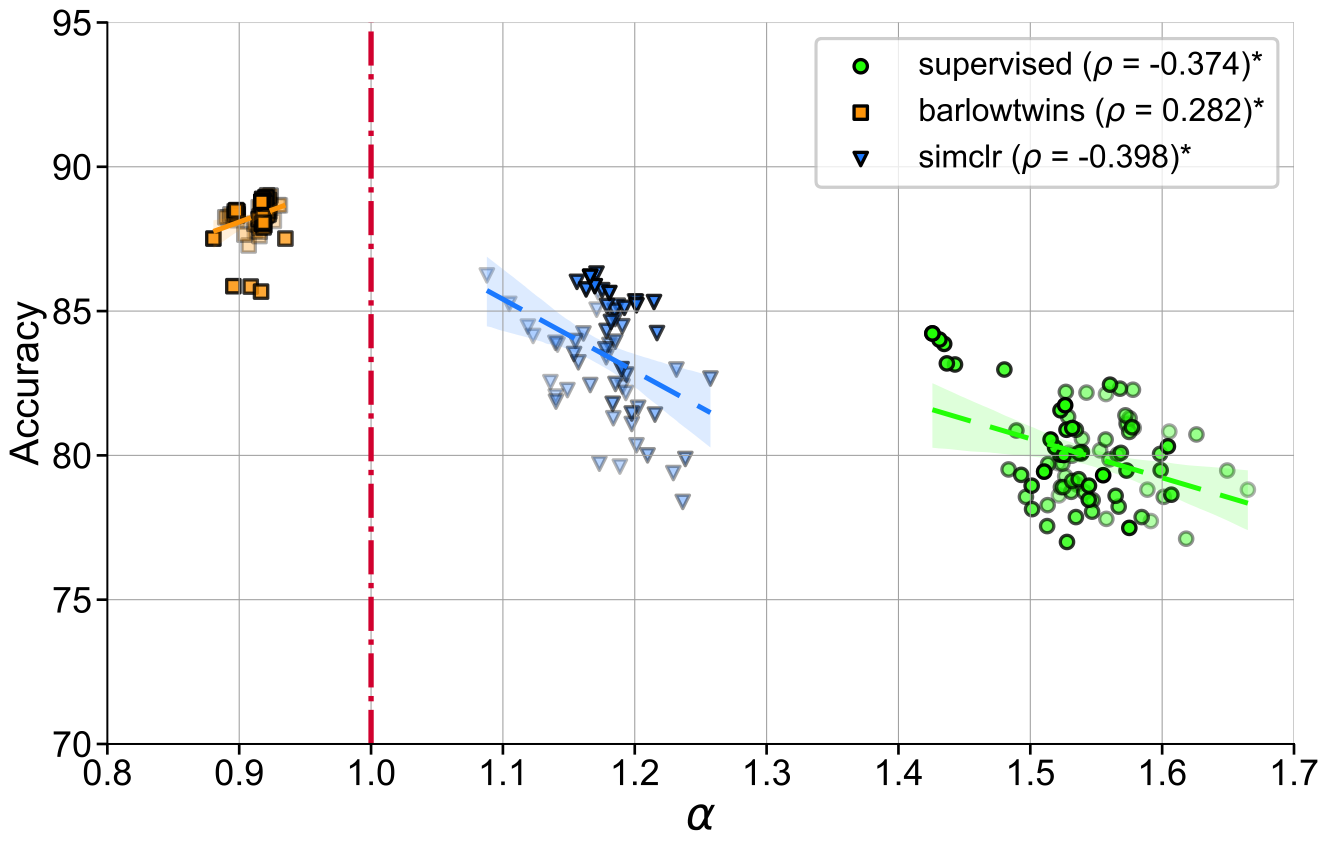}
\vspace{-2em}
\caption{Representations from ResNet-50/L75 finetuned for 30 epochs, on STL-10 dataset with \texttt{label\_noise=15\%}. Note that with increasing performance, $\alpha$ approaches 1 from either side of the spectrum (\textbf{solid points are later epochs}). Consequently, SSL objectives with representations s.t $\alpha$ closer to 1 compared to supervised training, are more robust (see \cref{table:robustness_STL10}). $^{*}p<0.05$. Check \cref{plot:finetune_STL10_appendix} for details.}
\label{plot:finetune_STL10}
\vspace{-1em}
\end{figure}

\section{Discussion}
\textbf{Summary} Our experiments suggest a strong correlation between decay coefficient for sample eigenspectrum of representations, $\alpha$, and generalization performance in tasks central to visual perception. In particular, from Figures \ref{fig:acc_alpha_arch} and \ref{fig:acc_alpha_loss} we note that across network architectures and pretraining objectives, classification accuracy on STL-10 improves as $\alpha$ approaches 1. Additionally, \cref{table:robustness_STL10} suggests that pretraining with self-supervised learning objectives provides representations which are relatively robust under noisy finetuning.

\textbf{Redundancy in ViT representations} 
Unlike all other models that were considered, representations from early layers of ViT had a rapid eigenspectrum decay with $\alpha > 1$ (see \cref{fig:acc_alpha_arch}).
The transformer architecture has a notable difference by design, i.e. early layers possess global receptive field context via self-attention on patch embeddings.  Raghu \textit{et al.} found that early ViT layers incorporate both local and global information \cite{raghu2021vision}. Based on these insights, one intuitive interpretation of our results is that the representations have a low effective rank, and encode redundant information relevant across multiple scales.

\textbf{Limitations} While the role of $\alpha$ and its relationship to generalization performance is better understood in the asymptotic setting for linear regression, similar questions in finite dimensional nonlinear models are unanswered. Moreoever, the complexity of computing eigenvalues scales $\mathcal{O}(D^3)$ where $D$ is the dimensionality of the representations. It is also worth noting that the empirical correlation was weaker for the earliest layers in each model. We believe that early layers learn more task invariant features such as corners and edges \cite{kornblith2019similarity,zeiler2014visualizing}, and therefore lack relevant information for downstream task. Therefore, $\alpha$ might not necessarily be reflective of generalization performance in poorly-trained models.

\textbf{Future Directions} Learning efficiently at scale from unlabelled datasets poses an exciting open problem in deep representation learning. We hope this work opens new perspectives on design of learning objectives and model architectures to learn task-agnostic features. Like smooth interpolation \cite{bubeck2021universal}, we hope that overparameterization and geometry in high-dimensions provide hints to a more principled understanding of generalization in deep neural networks.

\section*{Acknowledgement}
The authors would like to thank Zahraa Chorghay and Colleen Gillon for their aesthetic contribution to the manuscript and figures. This research was enabled in part by support provided by Mila (mila.quebec/en/) and Compute Canada (www.computecanada.ca).

\bibliography{paper}
\bibliographystyle{icml2022}

\newpage
\appendix
\onecolumn


\section{Proofs}
In this section, we present a formal proof of \cref{theorem_us_convergence}. In order to do so, we will use a lemma pertaining to iterative expression of the linear regression parameters over training epochs. This lemma is inspired by the results presented in \cite{shah2018minimum}

\begin{lemma}
Let $\hat{y} = x^T w$ be a finite dimensional linear regression problem where $w$ is learned using gradient descent in order to optimize the training error,
\begin{equation}
    \mathcal{L} = \mathbb{E}_{x,y}[(y - \hat{y})^2] = \mathbb{E}_{x,y}[(y - x^T w)^2] 
\end{equation}
where $(x,y) \sim \mathcal{D}_{train}$, i.e. the training dataset. Then $w_k$, i.e. $w$ after training for $k$ epochs can be written as
\begin{equation}
    w_k = X^T(XX^T)^{-1}\left[I-(I-\eta XX^T)^{k}\right]Y
\end{equation}
where $X,Y$ indicate the entire training dataset, i.e. $X \in \mathbb{R}^{N \times d}$ and $Y \in \mathbb{R}^N$
\label{lemma:A1}
\end{lemma}

\begin{proof}
We start with the gradient of the regression loss function for a defined training set denoted by $(X,Y)$ in a vectorized notation:
\begin{align}
    F(w_k) &= |Y-Xw_k|^2 \nonumber \\ 
    \implies \nabla F(w_k) &= X^T(Xw_k-Y) \\
\end{align}
We assume that the weights are initialized at 0, i.e. $w_0=0$. Using the gradient descent update:
\begin{align}
w_{k+1} &= w_k - \eta \nabla F = w_k - \eta X^T(Xw_k -Y) = (I - \eta X^TX)w_k + \eta X^TY \nonumber\\
w_1 &= (I-\eta X^TX)w_o + \eta X^TY = \eta X^TY \nonumber\\
\text{Let  }w_k &= \eta X^Tu_kY \implies u_1 = I \nonumber\\
w_2 &= \eta X^Tu_2Y \nonumber\\
    &= (I-\eta X^T X)\eta X^TY + \eta X^TY = \eta X^T[(I-\eta XX^T) + I]Y \nonumber\\
\implies u_2 &= (I-\eta XX^T) + I \nonumber\\
u_k &= (I-\eta XX^T)u_{k-1} + I = \sum_{i=0}^{k}(I-\eta XX^T)^{i-1}\nonumber\\
    &= (I-(I-\eta XX^T))^{-1}(I-(I-\eta XX^T)) \sum_{i=0}^{k}(I-\eta XX^T)^{i-1} \nonumber\\
    &= (\eta XX^T))^{-1}(I-(I-\eta XX^T)) \sum_{i=0}^{k}(I-\eta XX^T)^{i-1} \nonumber\\
    &= \frac{1}{\eta}(XX^T)^{-1} \sum_{i=0}^{k}[(I-\eta XX^T)^{i-1} - (I-\eta XX^T)^{i}]\nonumber\\
    &=\frac{1}{\eta}(XX^T)^{-1}[I - (I-\eta XX^T)^k]\nonumber\\
\implies w_k &= \eta X^Tu_kY = X^T(XX^T)^{-1}[I-(I-\eta XX^T)^k]Y 
\end{align}
\end{proof}

This proves \cref{lemma:A1}. We will now use this lemma to prove \cref{theorem_us_convergence}. From this result, we can also write:
\begin{equation}
    \Delta w_k = \eta X^T(I-\eta XX^T)^{k}Y
    \label{eq:lemma_A1_2}
\end{equation}

We restate the theorem from the main text. Note that the notations are simplified here from the theorem statement to improve readability.

\begin{theorem}
Let $\hat{Y} = X^T w$ be a finite dimensional linear regression problem where $w$ is learned using gradient descent. If we assume power law distribution in eigenspectrum of $X$, i.e. $\lambda_n = \frac{c}{n^{\alpha}} \forall n \in \{1,2 ... N\}$, then the time required by gradient descent to minimize the training error, $T_{convergence} = \mathcal{O}(N^{\alpha})$
\label{theorem:A2}
\end{theorem}

\begin{proof}
Using result from \cref{lemma:A1}, it is clear that the gradient converges to 0 if $\lambda_1<\frac{1}{\eta}$ where $\lambda_1$ is the leading eigenvalue of $XX^T$. 

Thus, $\eta < \frac{1}{\lambda_1} $, i.e. small learning rate setting.
So, we set $\eta = \frac{\hat{\eta}}{\lambda_1}$ where $\hat{\eta}<1$. Plugging this in \cref{eq:lemma_A1_2}
\begin{equation}
    \Delta w_k = \frac{\hat{\eta}}{\lambda_1}X^T(I- \frac{\hat{\eta}}{\lambda_1}XX^T)^kY
    \label{eq:delta_wk}
\end{equation}

Let $X=U\wedge^{\frac{1}{2}} V^T$ denote the singular value decomposition (SVD), which implies $XX^T = U\wedge U^T$. Using the SVD, we get $(I- \frac{\hat{\eta}}{\lambda_1}XX^T)^k = (I-\frac{\hat{\eta}}{\lambda_1}U\wedge U^T)^k$. It is worth noting that eigenvalues and eigenvectors of $(I-\frac{\hat{\eta}}{\lambda_1}U\wedge U^T)$ are related to that of $XX^T$ as shown below:
\begin{align}
    (I-\frac{\hat{\eta}}{\lambda_1}U\wedge U^T)u_i &= u_i - \frac{\hat{\eta}}{\lambda_1}(U\wedge U^Tu_i) \nonumber \\
    &= u_i - \frac{\hat{\eta}}{\lambda_1} \lambda_i u_i \quad \quad \text{[Using $U^TU = I$]} \nonumber \\
    \implies (I-\frac{\hat{\eta}}{\lambda_1}U\wedge U^T)u_i &= (1-\frac{\hat{\eta}}{\lambda_1} \lambda_i) u_i
    \label{eq:eigen_equiv}
\end{align}
Using \cref{eq:eigen_equiv}, we can write $(I-\frac{\hat{\eta}}{\lambda_1}U\wedge U^T)$ in the eigendecomposition form as $U\Tilde{\wedge}U^T$ where $\Tilde{\lambda_i} = 1-\hat{\eta}\frac{\lambda_i}{\lambda_1}$.
Thus, $\left(I-\frac{\hat{\eta}}{\lambda_1}U\wedge U^T\right)^k = U\Tilde{\wedge}^k U^T$. Plugging this in \cref{eq:delta_wk}, we get:
\begin{equation}
    \Delta w_k = \frac{\hat{\eta}}{\lambda_1}V\wedge^{\frac{1}{2}}U^TU\Tilde{\wedge}^kU^TY = \frac{\hat{\eta}}{\lambda_1}V\wedge^{\frac{1}{2}}\Tilde{\wedge}^kS
\end{equation}
where $S=U^TY\in \mathbb{R}^N$. For the $i^{th}$ element, we get $\Delta w_k^{(i)} = \frac{\hat{\eta}}{\lambda_1}\sum_{j}v_{i,j}\sqrt{\lambda_j}\Tilde{\lambda_j}^kS_j = \frac{\hat{\eta}}{\lambda_1}\sum_{j}v_{i,j}\sqrt{\lambda_j}(1-\hat{\eta}\frac{\lambda_j}{\lambda_1})^k S_j$. Since all other factors remain constant across training, i.e. do not change with $k$, the convergence of gradient descent depends on the factors $\left(1-\hat{\eta}\frac{\lambda_j}{\lambda_1}\right)^k$. Note that we define gradient descent to converge when $\Delta w_k^{(i)} \approx 0$ $\forall$ $i$. Therefore, the limiting factor that determines rate of convergence is $(1-\hat{\eta}\frac{\lambda_j}{\lambda_1})^k$, which in turn is limited by the smallest eigenvalue factor: $\frac{\lambda_N}{\lambda_1}$. \\
Assuming $\frac{\lambda_N}{\lambda_1} \ll 1 \implies \hat{\eta}\frac{\lambda_N}{\lambda_1} \ll 1$ as $\hat{\eta}<1 \implies (1-\hat{\eta}\frac{\lambda_N}{\lambda_1})^k \approx 1-k\hat{\eta}\frac{\lambda_N}{\lambda_1}$.\\ 
Hence the convergence time, $k^{*} = \mathcal{O}(\hat{\eta}\frac{\lambda_1}{\lambda_N}) = \mathcal{O}(\frac{\lambda_1}{\lambda_N}) $\\
If $\lambda_i$ follows power law, i.e, $\lambda_i=ci^{-\alpha}$ and $\frac{\lambda_N}{\lambda_1}=N^{-\alpha}$ then $k^{*} = \mathcal{O}(N^{\alpha})$ i.e. $k^{*}$ grows exponentially with $\alpha$.
\end{proof}

\section{Experimental results}


\begin{figure}[h!]
\centering
\includegraphics[width=0.7\columnwidth]{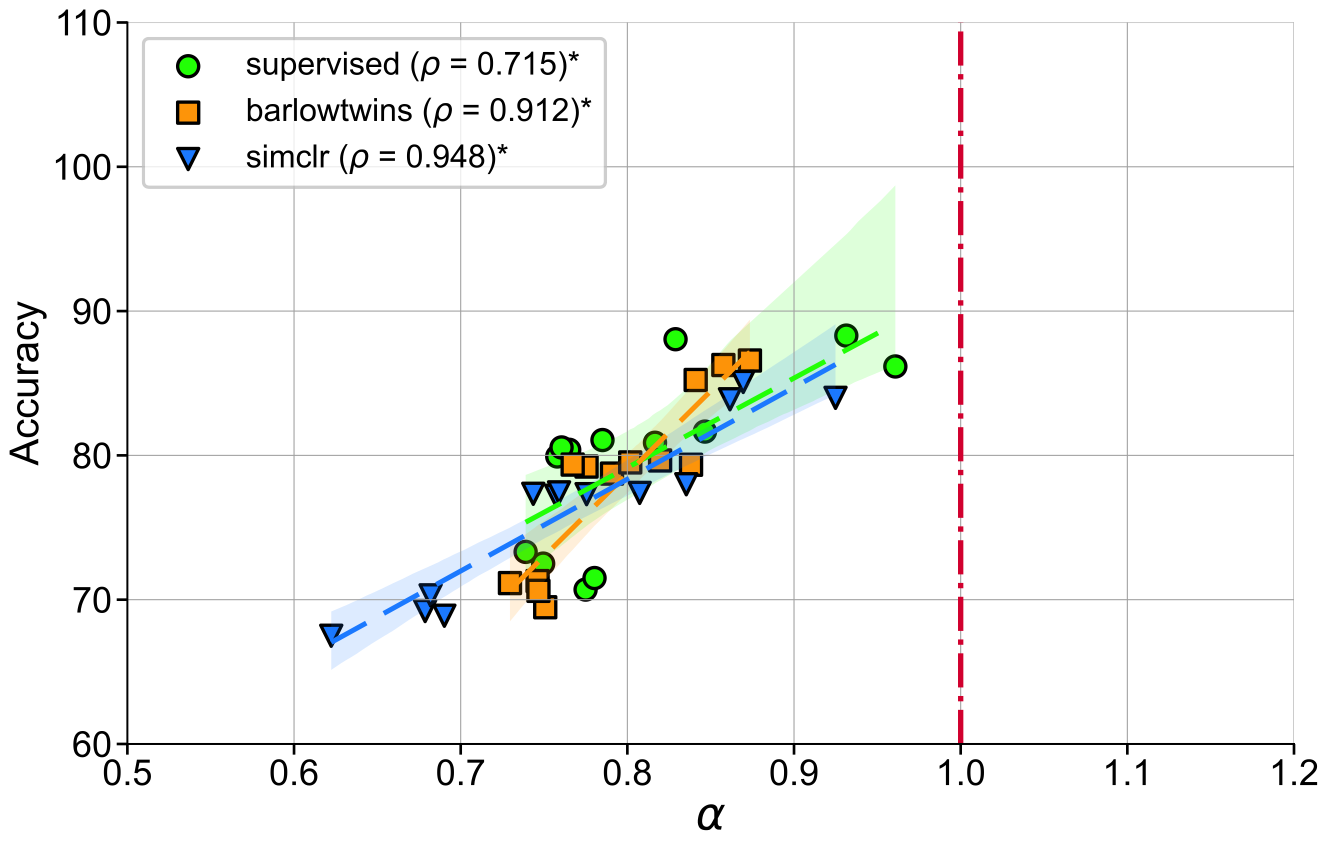}
\caption{Performance of non-linear readout on STL-10 is strongly correlated to proximity of $\alpha$ to 1 across different pretraining loss functions. The trends are similar to the performance of a linear readout, as shown in \cref{fig:acc_alpha_loss}. $^{*}p<0.05$.}
\label{fig:acc_alpha_loss_mlp}
\end{figure}

\begin{figure}[h!]
\centering
\includegraphics[width=0.8\columnwidth]{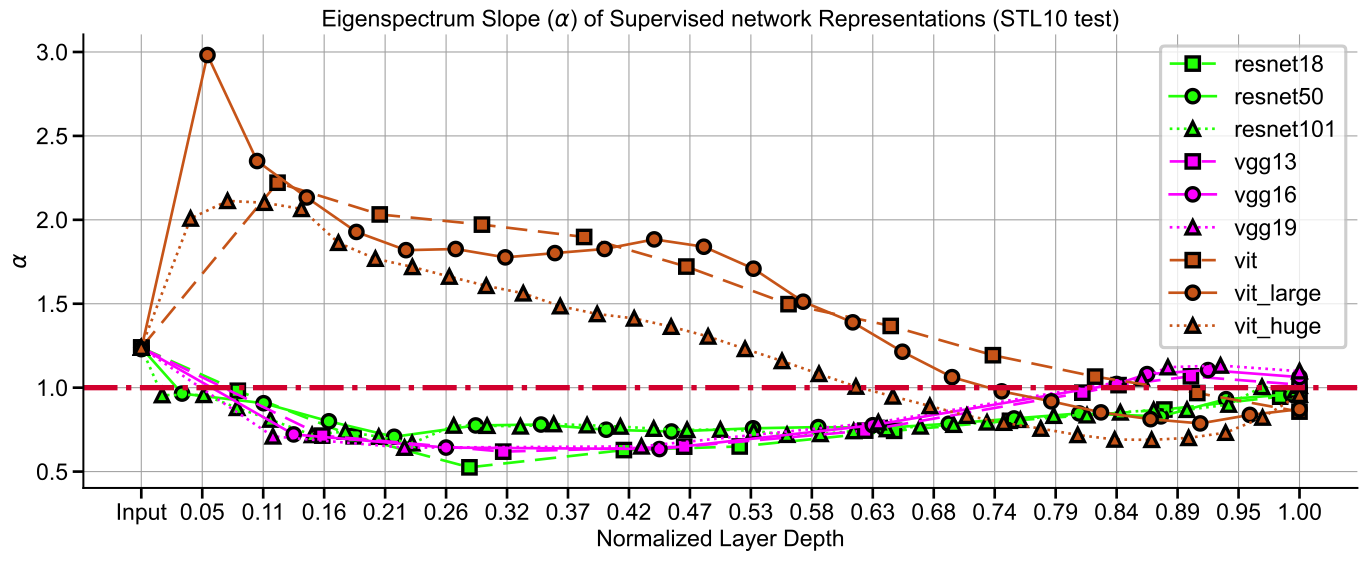}
\caption{$\alpha$ for intermediate layer representations from different backbone architectures demonstrates the contrasting representations learned by CNNs and ViT.}
\label{fig:alpha_layers_arch}
\end{figure}

\begin{figure}[h!]
\centering
\includegraphics[width=0.8\columnwidth]{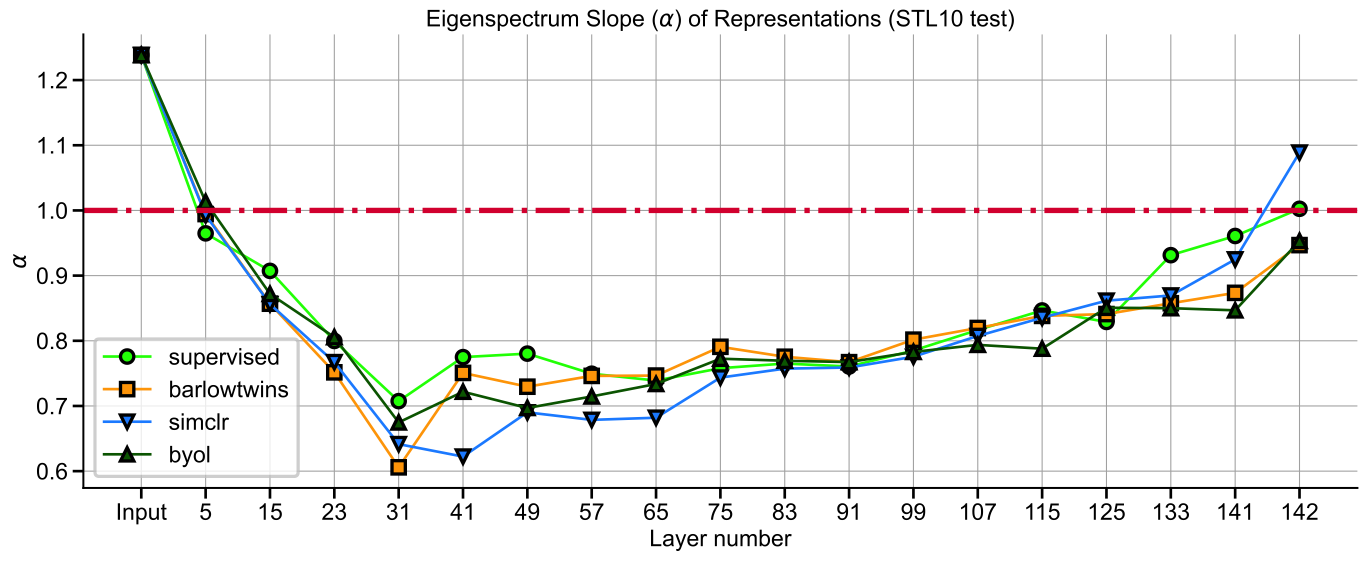}
\caption{$\alpha$ for intermediate layer representations from networks trained using different loss functions show similar trends. Representations from deeper layers exhibit $\alpha$ closer to 1 as compared to middle layer representations.}
\label{fig:alpha_layers_loss}
\end{figure}

\begin{figure}[h!]
\centering
\includegraphics[width=0.8\columnwidth]{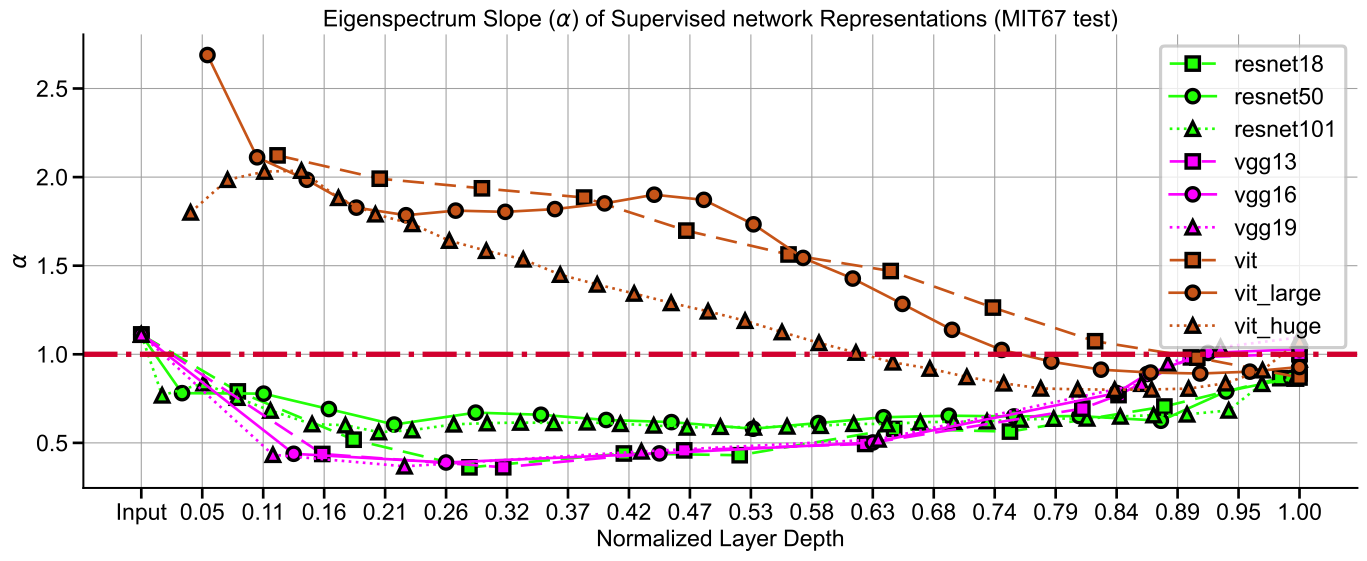}
\caption{$\alpha$ for intermediate layer representations from different backbone architectures in MIT67. Representations learned by ViT is qualitatively different from those is CNNs both in object and scene recognition datasets.}
\label{fig:alpha_layers_arch_MIT67}
\end{figure}

\begin{figure}[h!]
\centering
\includegraphics[width=0.8\columnwidth]{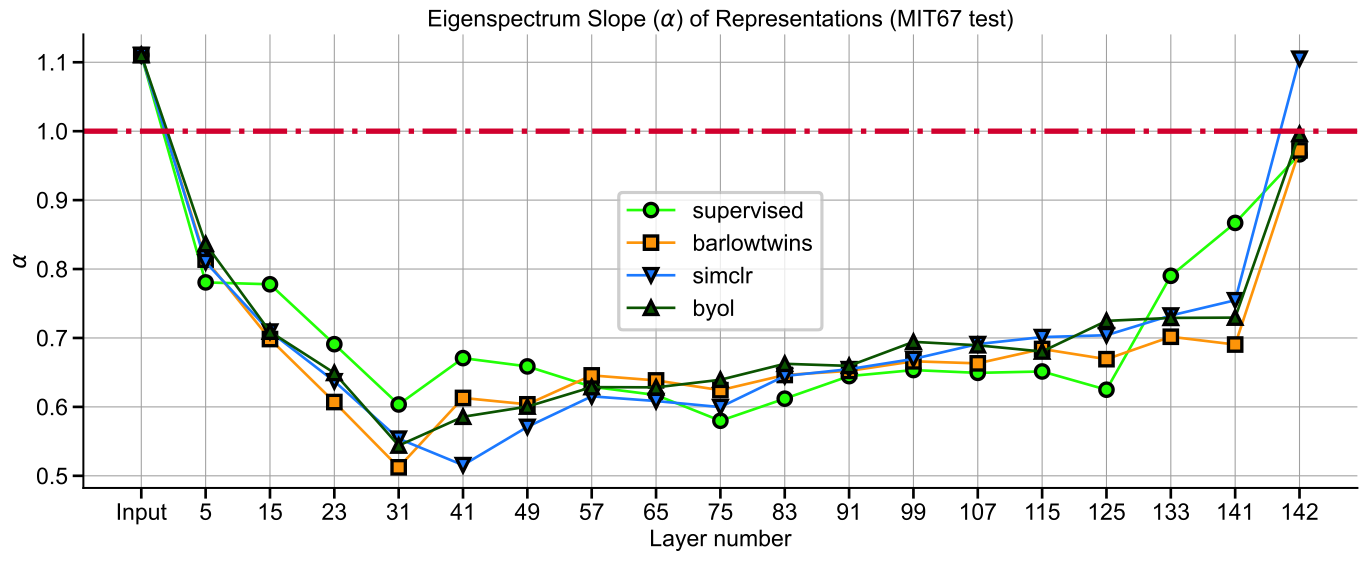}
\caption{$\alpha$ for intermediate layer representations from networks trained using different loss functions show similar trends in MIT67. Representations from deeper layers exhibit $\alpha$ closer to 1 as compared to intermediate layer representations.}
\label{fig:alpha_layers_loss_MIT67}
\end{figure}

\begin{figure}[t!]
\centering
\includegraphics[width=0.8\columnwidth]{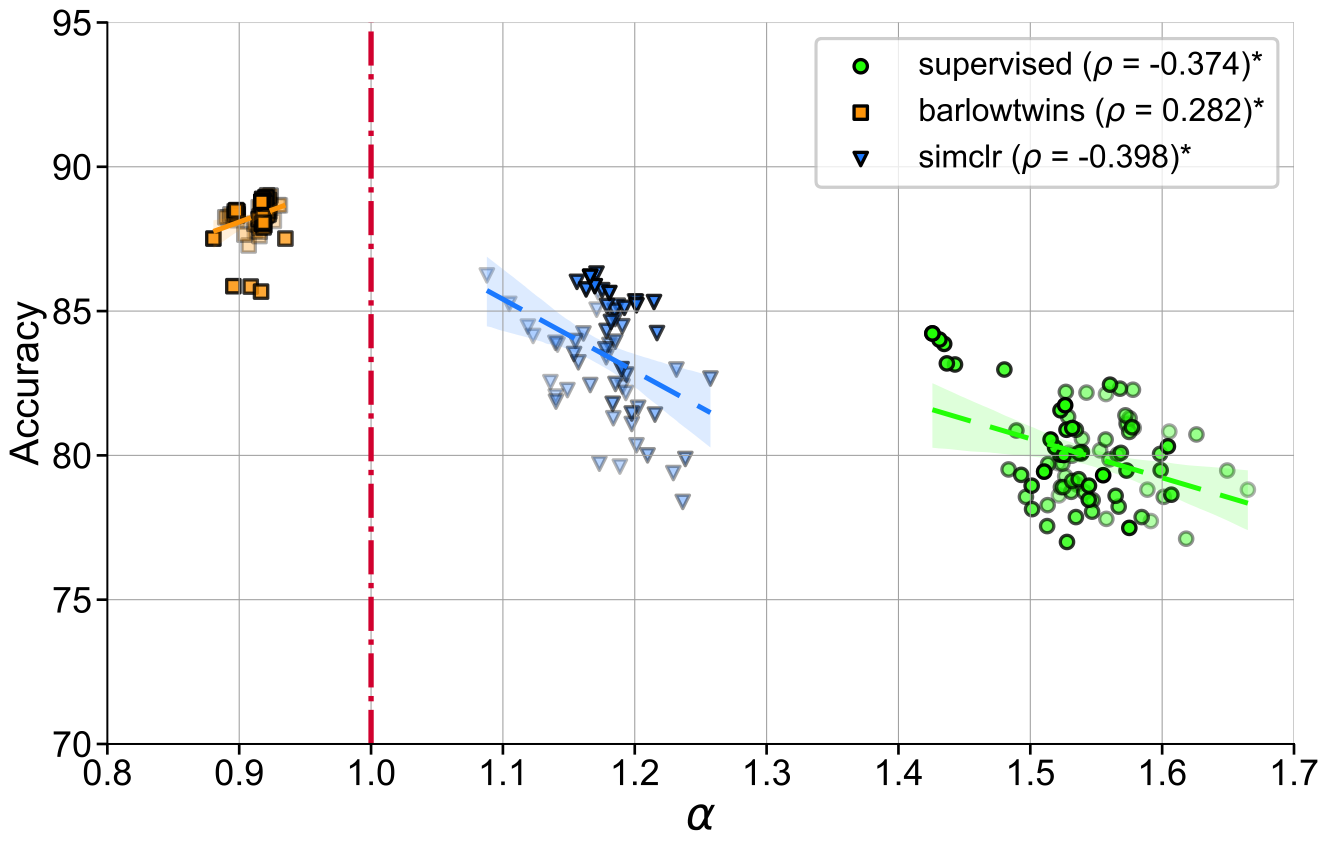}
\vspace{-2em}
\caption{Representations from ResNet-50/L75 finetuned for 30 epochs, on STL-10 dataset with \texttt{label\_noise=15\%}. Note that with increasing performance, $\alpha$ approaches 1 from either side of the spectrum (\textbf{solid points are later epochs}). Consequently, SSL objectives with representations s.t $\alpha$ closer to 1 compared to supervised training, are more robust (see \cref{table:robustness_STL10}). $^{*}p<0.05$}
\label{plot:finetune_STL10_appendix}
\vspace{-1em}
\end{figure}


\end{document}